\documentclass[letterpaper, 10 pt, conference]{ieeeconf}  

\IEEEoverridecommandlockouts                              

\usepackage{graphics} 
\usepackage{epsfig} 
\usepackage{mathptmx} 
\usepackage{times} 
\usepackage{amsmath} 
\usepackage{amssymb}  
\usepackage{cite}
\usepackage{bm}
\usepackage{yhmath}
\usepackage{color}
\usepackage{arcs}
\usepackage{algorithm}               
\usepackage{algorithmic}             

\renewcommand{\QED}{\QEDopen}

\newtheorem{proposition}{Property}
\newtheorem{lemmas}{Remark}

\newcommand*{\QEDA}{\hfill\ensuremath{\square}}

\title{\LARGE \bf
Essential Properties of Numerical Integration for Time-optimal Trajectory Planning Along a Specified Path
}

\author{Peiyao Shen, Xuebo Zhang and Yongchun Fang
\thanks{This work is supported in part by National Natural Science Foundation
of China under Grant 61573195 and Grant 613205017.}
\thanks{The authors are with the Institute of Robotics and Automatic Information System, Nankai University, Tianjin 300071, China,
and also with the Tianjin Key Laboratory of Intelligent Robotics, Nankai University, Tianjin 300071, China
        {\tt\small (e-mail: zhangxuebo@nankai.edu.cn)}}%
}

\begin{document}

\maketitle
\thispagestyle{empty}
\pagestyle{empty}

\begin{abstract}
This letter summarizes \emph{some known properties} and also presents \emph{several new properties} of the Numerical Integration (\emph{NI}) method for time-optimal trajectory planning along a specified path. The contribution is that \emph{rigorous mathematical proofs} of these properties are presented, most of which have not been reported in existing literatures. We first give some properties regarding switch points and accelerating/decelerating curves of the \emph{NI} method. Then, for the fact that when kinematic constraints are considered, the original version of \emph{NI} which only considers torque constraints may result in failure of trajectory planning, we give the concrete failure conditions with rigorous mathematical proof. Accordingly, a failure detection algorithm is given in a `run-and-test' manner. Some simulation results on a unicycle vehicle are provided to verify those presented properties. Note that though those \emph{known properties} are not discovered first, their mathematical proofs are given first in this letter. The detailed proofs make the theory of \emph{NI} more complete and help interested readers to gain a thorough understanding of the method.\\
\end{abstract}


\section{Introduce}

Due to low computational complexity, decoupled planning \cite{offline,OWMR_Kim,No_sm2}  becomes a popular motion planning method, which consists of two stages.
In the first stage, path planning methods are used to generate a geometric path with high level constraints including obstacle avoidance, curvature, and so on.
In the second stage, trajectory planning, which aims to assign a motion time profile to the specified path, could be then simplified as a planning problem in two-dimensional path parametrization space ($s, \dot s$), with $s$ and $\dot s$ being the path coordinate and path velocity respectively.
To improve  working efficiency, several time-optimal methods have been reported for the trajectory planning stage along a specified path, including \emph{Dynamic Programming} \cite{Dp1,Dp2,Dp3}, \emph{Convex Optimization} \cite{Convex1,Convex2,Convex3} and \emph{NI} \cite{NI_Shin,NI_Bobrow,NI_Improve,NI_Quang,NI_Pfeiffer,NI_Shiller,NI_Quang2,NI_Beh,No_sm1,NI_spacecraft,NI_human,NI_Zlajpah,NI_Lamiraux}.
{In addition, the work in \cite{Kal} proposes a real-time trajectory generation approach for omni-directional vehicles by constrained dynamic inversion.}

In this letter, we will focus on the \emph{NI} method, since its computational efficiency is shown in \cite{NI_Quang} to be better than other methods. The original version of \emph{NI} is proposed at almost the same time in the works \cite{NI_Shin,NI_Bobrow}, which aims to give a time-optimal trajectory along a specified path in the presence of only torque constraints for manipulators. Based on Pontryagin Maximum Principle, \emph{NI} possesses a bang-bang structure of torque inputs, thus, the core of \emph{NI} is the computation of switch points. The accelerating and decelerating curves, integrated from those switch points, constitute the time-optimal trajectory.
The work in \cite{NI_Improve} presents three types of switch points: tangent, discontinuity and zero-inertia points.
The detection methods for zero-inertia switch points are presented in \cite{NI_Quang,NI_Pfeiffer,NI_Shiller}.
Based on previous works, Pham \cite{NI_Quang} provides a fast, robust and open-source implementation for \emph{NI} in C++/Python, which is subsequently extended to the case of redundantly-actuated systems in the work \cite{NI_Quang2}.
In addition to manipulators \cite{NI_Beh, No_sm1}, \emph{NI} is also applied to spacecrafts \cite{NI_spacecraft}  and humanoid robots \cite{NI_human}.
In real applications, in addition to torque constraints, velocity constraints (bounded actuator velocity and path velocity) should also be considered.
Under these constraints, the works in \cite{NI_Zlajpah, NI_Lamiraux} indicate that the original \emph{NI} method with only torque constraints \cite{NI_Shin} can not be directly applied. {Considering velocity and acceleration constraints, the work \cite{No_sm1} focuses on detecting tangent, discontinuity and zero-inertia switch points on the speed limit curve decided by velocity constraints. However, rigorous proofs have not been reported to expose the conditions and underlying reasons of failure cases in aforementioned literatures.}


In this brief, we summarize some known properties and amend corresponding mathematical proofs which have not been reported in existing literatures; in addition, some new properties are excavated and proven rigorously. Some properties indicate the evolution of accelerating/decelerating curves and their intersection points. And another important property indicates that, when kinematic constraints are considered, the original version of \emph{NI} which only considers torque constraints, may result in failure of trajectory planning tasks. For this property, we first give the concrete failure conditions and detailed proofs. Accordingly, the failure detection algorithm is given in a `run-and-test' manner. Simulation results on a unicycle vehicle are provided to verify these properties.

{
The main contribution of this letter is the rigorous and detailed proofs for all presented properties:
\begin{enumerate}
  \item For \emph{Properties \ref{p2}-\ref{ps}} reported in \cite{AVP}, we first give their mathematical proofs in Section III-A.
  \item Some new properties are presented and proven, including \emph{Property \ref{p3}} in Section III-B and \emph{Properties \ref{t1}-\ref{t2}} in Section III-C.
\end{enumerate}
}

{These properties and proofs make the theory of \emph{NI} more complete and  help interested readers to gain an thorough understanding of the \emph{NI} method.}

The remainder of this letter is divided into four sections. Section II introduces notations and procedures of the \emph{NI} method.
Section III presents essential properties of the \emph{NI} method, and provides rigorous mathematical proofs.
In Section IV, simulation results are provided to verify the properties.
Finally, Section V gives some conclusions.

\section{Numerical Integration}

In this section, we briefly introduce the notations and procedures of the \emph{NI} method in \cite{NI_Shin,NI_Bobrow}.
{For the time-optimal path-constrained trajectory planning problem, based on Pontryagin Maximum Principle, the original \emph{NI} method \cite{NI_Shin,NI_Bobrow} uses a bang-bang structure of torque input to generate a time-optimal trajectory. First, torque constraints are converted to path acceleration constraints along the given path. Then, switch points of path acceleration are found. Finally, a time-optimal trajectory is integrated numerically with maximum and minimum path acceleration. A stable and open-source implementation of \emph{NI} method can be found in [13], and kinematic constraints are handled in the implementation with the method proposed in \cite{NI_Zlajpah}. Yet, in the presence of kinematic constraints, the rigorous proof showing conditions of failure of the original \emph{NI} with only torque constraints, has not been reported.}

\subsection{{Time-optimal Path-constrained Trajectory Planning}}

{The time-optimal path-constrained trajectory planning problem is to find a time-optimal velocity profile along the given path for a robot under various constraints such as torque constraints, velocity constraints, and so on. Note that the direction of the path velocity is supposed to be tangent to the given path.}

\subsection{Notations}

{In order to explain various notations clearer, we will refer to Fig. \ref{notation} in the following.}

\begin{figure}[thpb]
      \centering
      \includegraphics[scale=0.6]{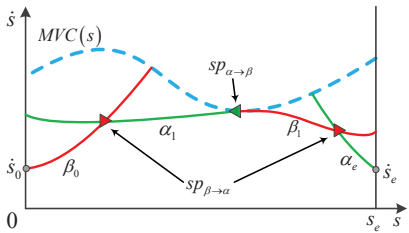}
      \caption{{Red solid curves $\beta_0,\beta_1$ represent accelerating $\beta\text{-}profiles$. Green solid curves $\alpha_1,\alpha_e$ represent decelerating $\alpha\text{-}profiles$. Red $\rhd$ and green $\lhd$  denote the points
$sp_{\beta\rightarrow\alpha}$ and $sp_{\alpha\rightarrow\beta}$, respectively. The scalars $\dot s_0, \dot s_e$ are respectively the starting and terminal path velocity at the endpoints of the given path. The $s_e$ is the total length of the given path.} }
      \label{notation}
\end{figure}

$s \ $: Path coordinate along a specified path.


$\dot s \ $: Path velocity along a specified path.

$\ddot s \ $: Path acceleration along a specified path.



$\bm A(s)\ddot s + \bm B(s)\dot s^2 + \bm C(s) \leq \bm 0$ \cite{NI_Shin}:
The inequality constraint along a specified path for $s, \dot s, \ddot s$, which is derived from torque constraints. {As an example, for an $n$-dof manipulator,
in order to guarantee torque constraints in a path-constrained trajectory planning task, the inequality constraint \cite{NI_Shin}}
\begin{equation}\label{ineq}
M(\bm \xi)\bm{\ddot \xi} + \bm{\dot \xi}^\mathrm{T}P(\bm \xi) \bm{\dot \xi} +\bm Q(\bm \xi) \leq \bm 0
\end{equation}
should hold, where the state $\bm \xi$ is an $n$-dimensional vector, $M$ is an $m\times n$ matrix, $P$ is an $n\times m\times n$ tensor and $\bm Q$ is an $m$-dimensional vector. Along the specified path, the robot state and its differentials are
\begin{equation}\label{state}
\bm \xi = \bm \xi(s), \ \bm{\dot \xi} = \bm \xi_s\dot s, \ \bm{\ddot \xi}=\bm \xi_{ss}\dot{s}^2 + \bm \xi_s\ddot s
\end{equation}
where $\bm \xi_s = d\bm \xi / ds$, $\bm \xi_{ss} = d\bm \xi_s / ds$. After substituting (\ref{state}) into (\ref{ineq}), we obtain that
\begin{equation}\label{abc}
\bm A(s)\ddot s + \bm B(s)\dot s^2 + \bm C(s) \leq \bm 0,
\end{equation}
with
\begin{equation} \nonumber
\begin{split}
& \bm A(s)=M(\bm \xi(s))\bm \xi_s(s), \\
& \bm B(s) = M(\bm \xi(s))\bm \xi_{ss}(s) + \bm \xi_s(s)^\mathrm{T}P(\bm \xi(s))\bm \xi_s(s),\\
& \bm C(s) = \bm Q(\bm \xi(s)).
\end{split}
\end{equation}

$\alpha(s, \dot s), \beta(s, \dot s):$ In order to
guarantee torque constraints, the scalars $s,\dot s,\ddot s$ should satisfy the inequality (\ref{abc}). Therefore, given the path coordinate $s$ and path velocity $\dot s$, the path acceleration $\ddot s$ satisfies the following inequality
\begin{equation}\label{alpbet}
\alpha(s,\dot s)\leq \ddot s \leq \beta(s,\dot s),
\end{equation}
where the minimum path acceleration $\alpha(s,\dot s)$ and maximum path acceleration $\beta(s, \dot s)$ are computed as
\begin{align}
& \alpha(s,\dot s) = \mathrm{max}\{\alpha_i|\alpha_i = \dfrac{-B_i(s)\dot s^2 - C_i(s)}{A_i(s)}, A_i(s)<0\},\label{alp} \\
& \beta(s,\dot s) = \mathrm{min}\{\beta_i|\beta_i = \dfrac{-B_i(s)\dot s^2 - C_i(s)}{A_i(s)}, A_i(s)>0\},\label{bet}
\end{align}
{wherein the integer $i\in[1,m]$, with $m$ being the dimension of the vector $\bm A$}. \textcolor{blue}{Please see the detailed description in \cite{NI_Bobrow}}.



$MVC$: The maximum velocity curve in the plane $(s,\dot s)$ is represented as
\begin{equation} \label{MVC}
MVC(s)=\mathrm{min}\{\dot s \geq 0 |\alpha(s,\dot s) = \beta(s, \dot s) \}, \  s\in[0,s_e].
\end{equation}
{For instance, the cyan dash curve in Fig. \ref{notation} is $MVC$.} If the robot state is on the $MVC$, there exists at
least one saturated actuator torque.

$AR$: The admissible region, in the plane $(s,\dot s)$, is enclosed by the curve $MVC$ and the lines $\dot s=0,s=0,s=s_e$. Within the $AR$ except for the boundary $MVC$, all actuator
torques are between lower and upper bounds ($\alpha(s,\dot s) < \beta(s,\dot s)$).

$\alpha\text{-}profile$: The decelerating curve, in the plane $(s,\dot s)$, is integrated backward with minimum acceleration $\alpha(s,\dot s)$  in (\ref{alp}). The slope $k_{\alpha}$ is computed as $k_\alpha = d\dot s / ds = \alpha(s, \dot s) / \dot s$.

$\beta\text{-}profile$: The accelerating curve, in the plane $(s,\dot s)$, is integrated forward with maximum acceleration $\beta(s, \dot s)$  in (\ref{bet}). The slope $k_{\beta}$ is computed as $k_\beta = d\dot s / ds = \beta(s, \dot s) / \dot s$.

$sp_{\alpha\rightarrow\beta}$ \cite{NI_Improve}: {Switch points from decelerating to accelerating curves, such as the green $\lhd$ in Fig. \ref{notation}.
They are on the $MVC$ curve, and there are three different types: tangent, discontinuity or zero-inertia points.} At tangent points, the slope $k_{mvc} = d \dot s/ds$ of $MVC$ is equal to $k_{\alpha}(=k_{\beta})$. At discontinuity points, $MVC$ is discontinuous. At zero-inertia points, at least one $A_i(s) = 0$ holds for the corresponding path coordinate $s$.

$sp_{\beta\rightarrow\alpha}$: {Switch points from accelerating to decelerating curves, such as the red $\rhd$ in Fig. \ref{notation}.}

\subsection{Numerical Integration Algorithm}

Procedures of the original version of \emph{NI} \cite{NI_Shin} which only considers torque constraints are given as follows:

\begin{description}
  \item[\textbf{NI-1.}]   In the plane $(s, \dot s)$, starting from $(s=0,\dot s = \dot s_0)$, the accelerating curve $\beta\text{-}profile$ is integrated forward  with maximum path acceleration $\beta(s,\dot s)$ until one of the following cases occurs:
                  \begin{itemize}
                  \item the curve $MVC$ is hit, and go to \textbf{NI-2};
                  \item the line $\dot s = 0$ is hit, and output that this path is not traversable;
                  \item the line $s = s_e$ is hit, and go to \textbf{NI-3}.
                  \end{itemize}

  \item[\textbf{NI-2.}] From the hitting point, searching forward along $MVC$, the first tangent, discontinuity or zero-inertia
                 point found is $sp_{\alpha\rightarrow\beta}$.
                 \begin{itemize}
                   \item If $sp_{\alpha\rightarrow\beta}$ is detected, from the switch point, an $\alpha\text{-}profile$ is integrated backward with $\alpha(s,\dot s)$ until it intersects the generated $\beta\text{-}profile$ in \textbf{NI-1, NI-2} at a point $sp_{\beta\rightarrow\alpha}$, and
                  one new $\beta\text{-}profile$ is integrated forward as \textbf{NI-1}.
                   \item If $sp_{\alpha\rightarrow\beta}$ is not detected, go to \textbf{NI-3}.
                 \end{itemize}

  \item[\textbf{NI-3.}]  Starting from $(s=s_e,\dot s = \dot s_e)$, the decelerating curve $\alpha\text{-}profile$ is integrated backward with minimum path acceleration $\alpha(s,\dot s)$ until it intersects the generated $\beta\text{-}profile$ in \textbf{NI-1, NI-2} at a point $sp_{\beta\rightarrow\alpha}$. Finally, output a time-optimal trajectory consisting of those generated accelerating and decelerating curves in \textbf{NI-1, NI-2} and \textbf{NI-3}.
\end{description}

{Note that chattering or vibration is usually caused by high-frequent switching of acceleration. Fortunately, the work in \cite{NI_Shin} has indicated that switch points are finite for the \emph{NI} method, which generally does not cause severe chattering phenomenons.}

\section{Properties}

In this section, essential properties of \emph{NI} are provided with rigorous mathematical proofs. Note that \emph{Property \ref{p1}} has been presented and proven in \cite{AVP}; \emph{Properties \ref{p2}-\ref{ps}} have been reported in \cite{AVP}, and we will amend their mathematical proofs; \emph{Properties \ref{p3}-\ref{t1}} are first exposed in this letter with rigorous proofs; in the presence of kinematic constraints, the failure of the original \emph{NI} with only torque constraints is reported in \cite{NI_Zlajpah,NI_Lamiraux}, and we will give the mathematical conditions and underlying reasons of failure cases in \emph{Property \ref{t2}}. {See the proofs of \emph{Properties \ref{p2}-\ref{t1}} in Appendix A-D, respectively.}

\subsection{{Property of $\alpha\text{-}profiles$ and $\beta\text{-}profiles$}}

In the admissible region except for the boundary $MVC$ ($AREM$), the inequality $\alpha(s,\dot s) < \beta(s,\dot s)$ holds, which indicates $k_{\alpha} < k_{\beta}$ for each point. Three known properties are summarized in the following, and we will give rigorous mathematical proofs for \emph{Properties \ref{p2}-\ref{ps}}, which have not been reported in existing literatures.

\begin{proposition}\label{p1}
In the $AREM$ region, any two $\alpha\text{-}profiles$ never intersect with each other. Neither do $\beta\text{-}profiles$. (Please see the proof in \cite{AVP}.)
\end{proposition}


\begin{proposition}\label{p2}
In the $AREM$ region, if an $\alpha\text{-}profile$ intersects another $\beta\text{-}profile$ at a point $(s=s_c,\dot s = \dot s_c)$,
in terms of path velocity, the $\alpha\text{-}profile$ is greater than the $\beta\text{-}profile$ in the left neighborhood of $s_c$, but less than
the $\beta\text{-}profile$ in the right neighborhood of $s_c$.
\end{proposition}

\begin{proposition}\label{ps}
In the $AREM$ region, an $\alpha\text{-}profile$ is not tangent to another $\beta\text{-}profile$ at any point.
\end{proposition}

\subsection{Property of $sp_{\beta\rightarrow\alpha}$}

The point $sp_{\beta\rightarrow\alpha}$ is the intersection point between $\alpha\text{-}profile$ and $\beta\text{-}profile$, denoted as the red $\rhd$ in Fig. \ref{con5}. In the iterative process of \emph{NI} (see Section II-C), a $\beta\text{-}profile$ may intersect a finite number of integrated backward $\alpha\text{-}profiles$ \cite{NI_Shin}. Which one of these intersection points is finally chosen as $sp_{\beta\rightarrow\alpha}$ on the $\beta\text{-}profile$? The answer is given in \emph{Property \ref{p3}}, which is first presented in this letter with rigorous proof.

\begin{figure}[t]
      \centering
      \includegraphics[scale=0.6]{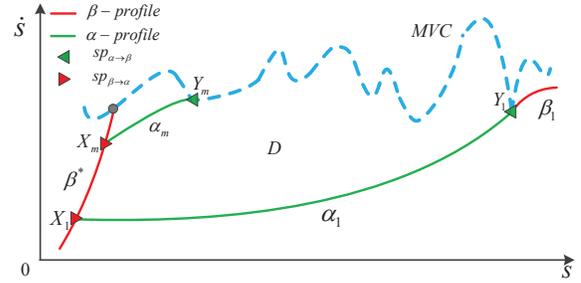}
      \caption{$m>1:$ Intersection points $X_i$, $i\in[1,m]$ are on $\beta^*$, wherein $X_j$, $1 < j < m $ is on \overarc[0.55]{$X_1X_m$}. Each $X_i$ has one corresponding decelerating curve $\alpha_i$ and one switch point $Y_i$. The region $D$ is enclosed by $\alpha_1$, $\beta^*$ and $MVC$.}
      \label{con5}
\end{figure}
\begin{figure}[t]
      \centering
      \includegraphics[scale=0.6]{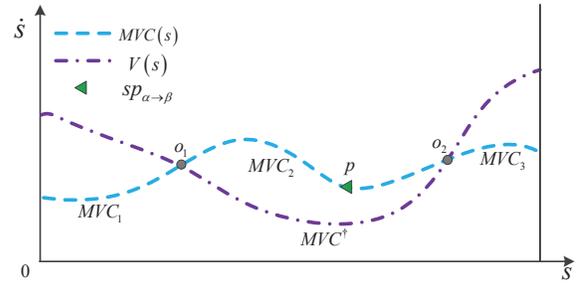}
      \caption{The maximum velocity curve is altered from  $MVC=MVC_1+MVC_2+MVC_3$ to $MVC^*=MVC_1+MVC^\dagger+MVC_3$.}
      \label{MVCt}
\end{figure}

\begin{proposition}\label{p3}
If one $\beta\text{-}profile$ intersects $m \geq 1$ $\alpha\text{-}profiles$, respectively at points $X_i, i\in [1,m]$, and in terms of path coordinate, $X_i$ is less than $X_j$, $1 \leq i < j \leq m$,
then, $X_1$ is finally chosen as $sp_{\beta\rightarrow\alpha}$ on the $\beta\text{-}profile$ after finite iterations.
\end{proposition}

\begin{lemmas}
In Fig. \ref{con5}, $Y_1$ may be the terminal point $(s_e,\dot s_e)$. In this special situation, \emph{Property \ref{p3}} still holds. The case $m = 1$ obviously holds. Meanwhile, the case $m>1$ also holds since the trajectory starting from $X_i,i>1$ cannot arrive at the terminal point according to \emph{Properties \ref{p1}-\ref{ps}}. In addition, other points $Y_i$, $i > 1$ cannot be the terminal point $(s_e,\dot s_e)$ since \emph{NI} has completed all procedures at the terminal point (see the procedure \textbf{NI-3} in Section II-C).\QEDA
\end{lemmas}

The \emph{Properties \ref{p1}-\ref{p3}} show the evolution of accelerating/decelerating curves and their intersection points.

\subsection{Property of NI with kinematic constraints}

The original \emph{NI} method \cite{NI_Shin} generates a time-optimal trajectory with bounded torque. However, when velocity constraints (bounded actuator velocity and path velocity) are taken into account, the original \emph{NI} (as shown in Section II-C) may result in a failure. Several examples for this property are shown in the work in \cite{NI_Zlajpah,NI_Lamiraux}, \emph{but the concrete failure conditions and the detailed proof have not been reported.} In the subsequent \emph{Property \ref{t2}}, we will elaborate the failure conditions and provide the mathematical proof.
Note that \emph{Property \ref{t1}} is first presented and proven to indicate the existence of tangent switch points on the limit curve considering velocity constraints, which is used in the proof of \emph{Property \ref{t2}}.

Actuator velocity constraints are transformed into path velocity constraints by kinematic models of robots,
therefore, velocity constraints are represented as
\begin{equation}\label{V}
V(s):s\rightarrow \dot s , \ \ s\in [0,s_e],
\end{equation}
where the scalar $V(s)$ is the maximum path velocity.
Due to the velocity constraints, the maximum velocity curve is altered as
\begin{equation}\label{MVCSTAR}
MVC^*(s)=\mathrm{min}(MVC(s),V(s)), \ \ s\in [0,s_e].
\end{equation}
In the region enclosed by $MVC^*, s= 0, s=s_e,\dot s = 0$, accelerating and decelerating curves
satisfy all velocity and torque constraints.
In addition, the part of $MVC^*$, which is different from $MVC$, is represented as
\begin{equation}\label{MVCtt}
MVC^\dagger(s)\!=\!\{MVC^*(s)\!|\!MVC^*(s)\!<\!MVC(s),s\in[0,s_e]\}.
\end{equation}
For instance, the dash-dot curve \overarc{$o_1o_2$} is $MVC^\dagger$ in Fig. \ref{MVCt}.

The maximum velocity curve altering from $MVC$ to $MVC^*$ results in a decreasing number of tangent switch points.
For instance, the tangent switch point $p$ disappears due to $MVC^\dagger$ instead of $MVC_2$ in Fig. \ref{MVCt}.

\begin{proposition}\label{t1}
Tangent switch points are nonexistent on $MVC^\dagger$.
\end{proposition}

The \emph{NI} method \cite{NI_Shin} with torque constraints, generates a time-optimal trajectory $T$. After considering velocity constraints $V(s)$, the maximum velocity curve is altered from $MVC$ to $MVC^*$, which causes that tangent switch points on the part $MVC^\dagger$ of $MVC^*$ are nonexistent according to \emph{Property \ref{t1}}. It has negative effects on searching switch points in the \emph{NI} method, and may result in failure of trajectory planning tasks.

\begin{figure}[thpb]
      \centering
      \includegraphics[scale=0.6]{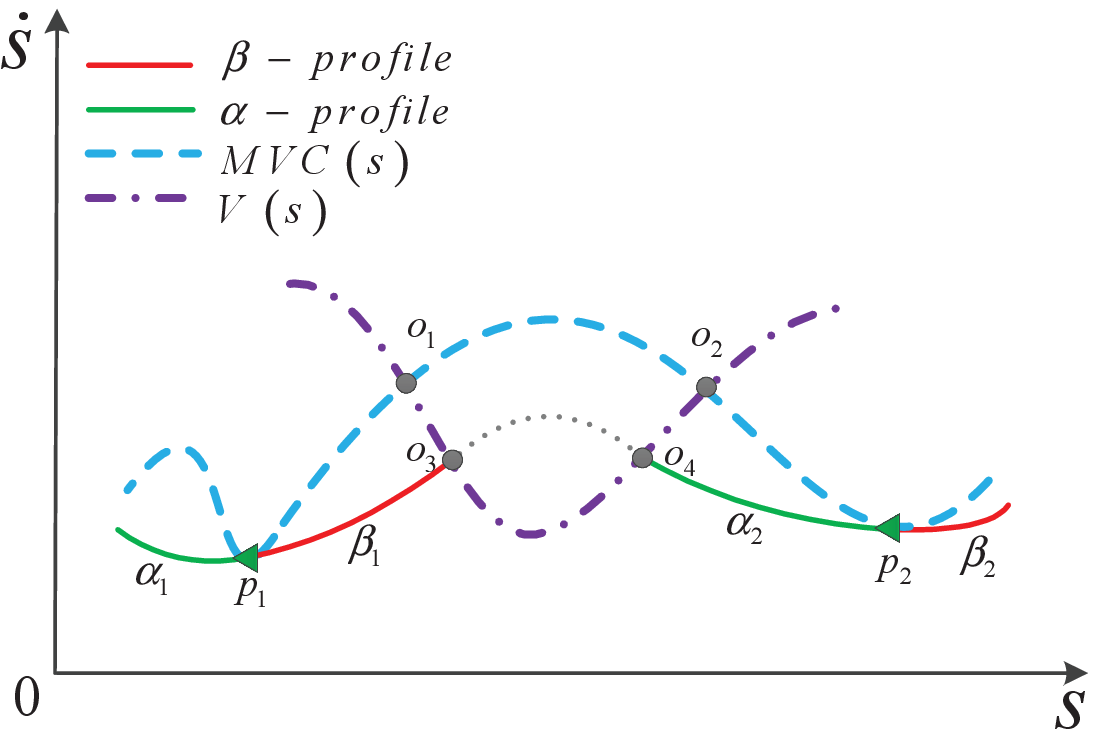}
      \caption{Left: $\beta\text{-}profile$, Right: $\alpha\text{-}profile$ }
      \label{wr1}
      \centering
      \includegraphics[scale=0.6]{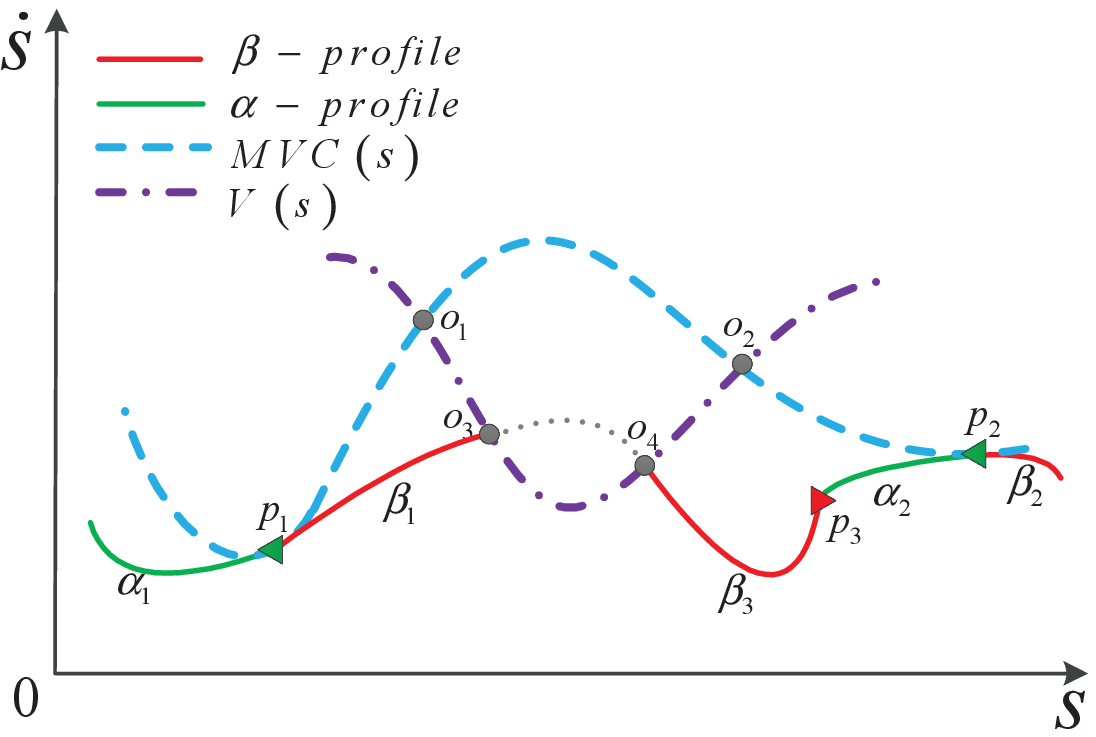}
      \caption{Left: $\beta\text{-}profile$, Right: $\beta\text{-}profile$ }
      \label{wr2}
      \centering
      \includegraphics[scale=0.6]{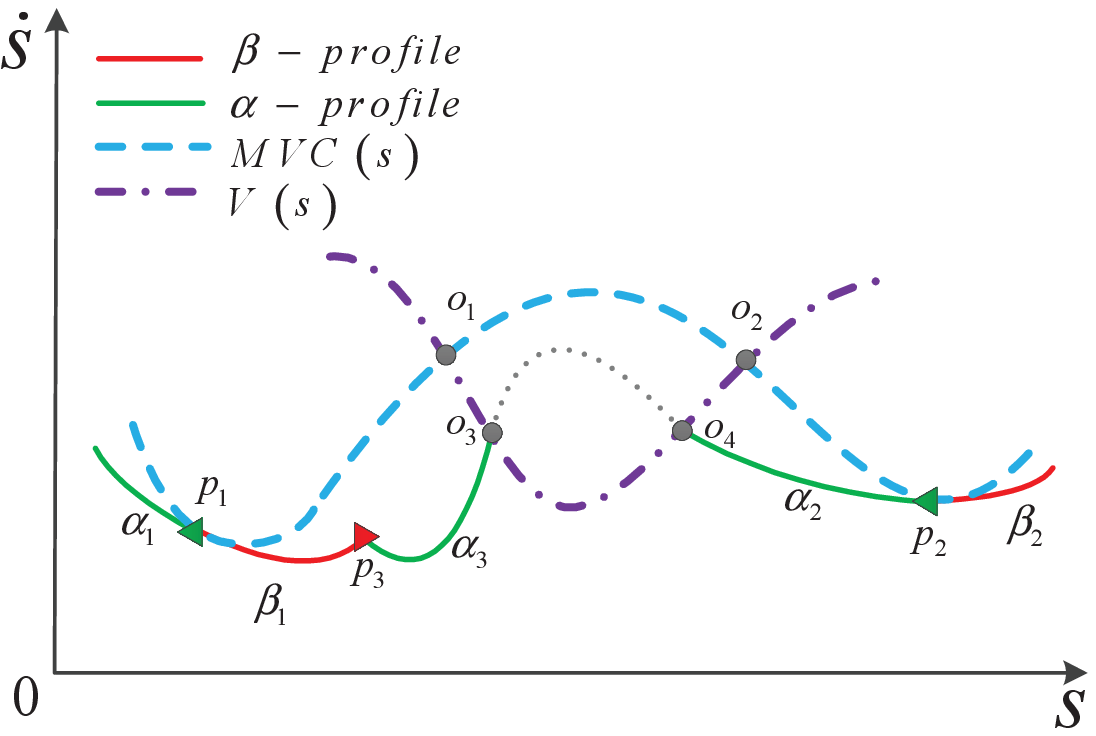}
      \caption{Left: $\alpha\text{-}profile$, Right: $\alpha\text{-}profile$ }
      \label{wr3}
      \centering
      \includegraphics[scale=0.6]{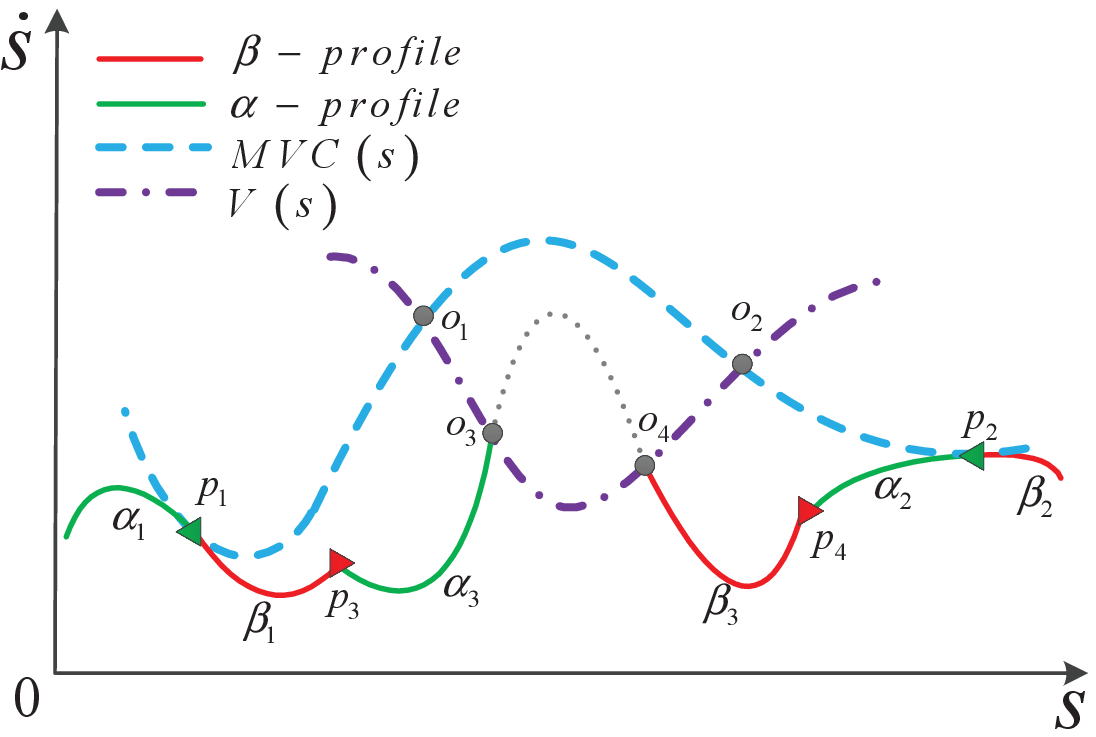}
      \caption{Left: $\alpha\text{-}profile$, Right: $\beta\text{-}profile$ }
      \label{wr4}
\end{figure}

\begin{proposition}\label{t2}
If the following conditions hold:
\begin{description}
  \item[$C_1:$] $\exists s^* \in [0,s_e], \ \ MVC^*(s^*) < T(s^*)$,
  \item[$C_2:$] discontinuity and zero-inertia switch points on $MVC^\dagger$ are nonexistent,
\end{description}
then, the \emph{NI} method using $MVC^*$ fails to find out a feasible trajectory.
\end{proposition}

\begin{proof}
The proof is proceeded in 4 steps.

\emph{Step 1: To show four essential cases of $C_1$.}

The time-optimal trajectory $T$, which is obtained from the \emph{NI} method using $MVC$, consists
of several $\alpha\text{-}profiles$ and $\beta\text{-}profiles$, satisfying $T(s) \leq MVC(s), s\in[0,s_e]$. According to (\ref{MVCtt}) and $C_1$, there must exist one part of $MVC^\dagger$ below $T$, and this part is called as $MVC^\ddag$. On both sides of $MVC^\ddag$, the trajectory $T$ could be accelerating or decelerating, thus there are four essential cases for the condition $C_1$ as shown in Figs. \ref{wr1}-\ref{wr4}.
The purple dash-dot line \overarc[0.65]{$o_3o_4$} represents $MVC^\ddag$.
The gray dotted line $\overline{o_3o_4}$ and red/green solid lines constitute the trajectory $T$.
The green $\lhd$ and red $\rhd$ are $sp_{\alpha\rightarrow\beta}$ and $sp_{\beta\rightarrow\alpha}$ points respectively.
These essential cases are regarded as basic components of other complex cases, thus, for clarify, the proof of complex cases will be presented in \emph{Remark 2}.

\emph{Step 2: To prove that for each essential case, the NI method using $MVC^*$ must run to the procedure} \textbf{NI-2} \emph{in Section II-C, which is to search forward switch point along $MVC^\dagger$ (the purple dash-dot curve \overarc[0.65]{$o_1o_2$}) from the hitting point}.

Let $p_1$ be the neighborhood $sp_{\alpha\rightarrow\beta}$ of $T$ on the left side of $MVC^\dagger$, which also can be $p_1=(0,\dot s_0)$. Starting from $p_1$, \emph{NI} integrates $\beta_1$ forward. For essential cases as Figs. \ref{wr1}-\ref{wr2}, the accelerating curve $\beta_1$ hits $MVC^\dagger$, then \emph{NI} will search forward $sp_{\alpha\rightarrow\beta}$ along $MVC^\dagger$ from the hitting point of \overarc[0.65]{$o_1o_3$}. For essential cases as Figs. \ref{wr3}-\ref{wr4},
all accelerating curves including $\beta_1$ and $\beta\text{-}profiles$ from \overarc[0.65]{$p_1o_1$} cannot intersect $\alpha_3$ for \emph{Property \ref{p2}}, thus, it must intersect \overarc[0.65]{$p_1o_1$} or \overarc[0.65]{$o_1o_3$}. If the hit curve is \overarc[0.65]{$o_1o_3$}, then \emph{NI} will search forward $sp_{\alpha\rightarrow\beta}$ from the hitting point along
$MVC^\dagger$. If the hit curve is \overarc[0.65]{$p_1o_1$}, then \emph{NI} will search forward $sp_{\alpha\rightarrow\beta}$ from the hitting point along \overarc[0.65]{$p_1o_1$}. The number of switch points on \overarc[0.65]{$p_1o_1$} is finite \cite{NI_Shin},
therefore switch points on \overarc[0.65]{$p_1o_1$} will run out, and \emph{NI} will go on searching forward $sp_{\alpha\rightarrow\beta}$ along $MVC^\dagger$.

\emph{Step 3: To prove that those $\alpha\text{-}profiles$, starting from the right side of $o_4$, cannot intersect $\beta\text{-}profiles$, which are on the left side of $o_3$ and generated in iterative process of NI.}

Based on \emph{Property \ref{t1}} and condition $C_2$ of \emph{Property \ref{t2}}, switch points on $MVC^\dagger$ are nonexistent. Therefore, \emph{NI} will go on searching forward switch points on the right side of $o_2$. In Figs. \ref{wr1}-\ref{wr4}, the integrated backward $\alpha\text{-}profile$, starting from the $sp_{\alpha\rightarrow\beta}$ of \overarc[0.65]{$o_2p_2$}, cannot intersect $\alpha_2,\beta_3$ according to \emph{Properties \ref{p1}-\ref{p2}}, thus, it must hit the purple dash-dot \overarc[0.65]{$o_4o_2$} or cyan dash \overarc[0.65]{$o_2p_2$}.

Moreover, in the admissible region, starting from the right side of $p_2$, integrated backward $\alpha\text{-}profiles$ cannot intersect those $\beta\text{-}profiles$ generated by \emph{NI} and on the left side of $o_3$, which is proven by contradiction. Assume that in the admissible region, from the right side of $p_2$, an integrated backward $\alpha\text{-}profile$ intersects one of those $\beta\text{-}profiles$. Then, the $\alpha\text{-}profile$ should be part of the trajectory $T$ based on \emph{Property \ref{p3}}. However, in terms of path velocity, the $\alpha\text{-}profile$ is less than $MVC^*$, which contradicts with  $C_1$. Therefore, this assumption is invalid, and $\alpha\text{-}profiles$ from the right side of $p_2$ cannot intersect those $\beta\text{-}profiles$ generated on the left side of $o_3$.

In addition, if $p_2$ is the terminal point $(s_e,\dot s_e)$, then starting from $p_2$ and switch points at the right side of $o_2$, all integrated backward $\alpha\text{-}profiles$ must hit the purple dash-dot \overarc[0.65]{$o_4o_2$} or cyan dash \overarc[0.65]{$o_2p_2$} according to \emph{Properties \ref{p1}-\ref{ps}}, which also indicates that this step holds.

\emph{Step 4: To synthesize above analysis and indicate failure of trajectory planning tasks.}

The above \emph{Step 1-3} indicate that if conditions of \emph{Property \ref{t2}} hold, then the \emph{NI} method using $MVC^*$ fails to output a feasible trajectory:
in the iterative process,
$\alpha\text{-}profiles$ and $\beta\text{-}profiles$ on two different sides of $MVC^\ddag$ do not intersect in the admissible region, which causes that the final trajectory is incomplete.
In summary, \emph{Property \ref{t2}} holds.
\end{proof}


\begin{lemmas}
{For other complex cases mentioned in \emph{Step 1}, there may exist many parts of $MVC^\dagger$ below \emph{T}. The trajectory \emph{T} could be accelerating or decelerating at the first and last part as \emph{Step 1}. In these cases, \emph{NI} still searches switch points along $MVC^\dagger$ from the first part in the same reason as \emph{Step 2}. The condition  $C_2$ indicates that the $sp_{\alpha\rightarrow\beta}$ on $MVC^\dagger$ is nonexistent. Starting from switch points at the right side of the last part, all $\alpha\text{-}profiles$ cannot intersect $\beta\text{-}profiles$ at the left side of the first part as \emph{Step 3}. Therefore, in these complex cases, $\alpha\text{-}profiles$ and $\beta\text{-}profiles$ on two different sides of $MVC^\dagger$ do not intersect in the admissible region. It indicates that \emph{Property \ref{t2}} still holds for complex cases.} \QEDA
\end{lemmas}

In the presence of velocity and torque constraints, \emph{Property \ref{t2}} actually gives  sufficient conditions for failure of the \emph{NI} method in \cite{NI_Shin}, which is important because it is theoretically shown that the failure cases indeed exists for this method.
In the following, a necessary and sufficient failure condition is given by a numerical `run-and-test' algorithm (\emph{RT}):

\begin{description}
  \item[\textbf{RT-1.}] In the plane $(s,\dot s)$, the curve $MVC^*(s)$ is obtained with (\ref{MVCSTAR}). It is initialized that the point $p$ is $(0,\dot s_0)$, boolean variable $isContinuous$ is $TRUE$, the longest continuous trajectory $T^*$ from $(0,\dot s_0)$ is null and the scalar $sLast$ is zero. Then go to \textbf{RT-2}.
  \item[\textbf{RT-2.}] Starting from the point $p$, a $\beta\text{-}profile$ is integrated forward until it hits the boundaries of the admissible region at $(s= s_h,\dot s = \dot s_h)$.
      If $isContinuous$ is $TRUE$, then \emph{RT} will call the subfunction $addProfile$($T^*$, $\beta\text{-}profile$), and $sLast$ is updated as $s_h$.
      And go to \textbf{RT-3}.
  \item[\textbf{RT-3.}]
      From the hitting point, the first switch point found along $MVC^*$ or terminal point is assigned to the point $q$. Starting from $q$, an $\alpha\text{-}profile$ is integrated backward until it intersects  one of following lines:
       \begin{itemize}
         \item The trajectory $T^*$, then \emph{RT} will call the subfunction $addProfile$($T^*$, $\alpha\text{-}profile$) and set $isContinuous$$=$$TRUE$. Meanwhile, $sLast$ is updated as the path coordinate of $q$.
         \item The boundaries of the admissible region, then \emph{RT} will set $isContinuous=FALSE$.
       \end{itemize}
       The point $p$ is updated as $q$. If $p$ is the terminal point, then go to \textbf{RT-4}, else go to \textbf{RT-2}.
  \item[\textbf{RT-4.}]
        If $sLast < s_e$, this method fails to generate a feasible trajectory, else output a feasible trajectory $T^*$.
\end{description}

\begin{lemmas}
The subfunction $addProfile$ adds accelerating/decelerating curves to $T^*$. The core of this algorithm is running the \emph{NI} method and detecting whether those generated accelerating and decelerating curves constitute a continuous trajectory on the whole path. The scalar $sLast$ indicates the path length of the trajectory $T^*$. Therefore, after the algorithm is completed, the inequality $sLast < s_e$ in \textbf{RT-4} is a necessary and sufficient failure condition for the \emph{NI} method in presence of velocity and torque constraints.\QEDA
\end{lemmas}

\section{Simulation Results}

In order to verify \emph{Properties \ref{p1}-\ref{t2}}, some simulation results on a unicycle vehicle are provided in this section.

\subsection{Model}
{A unicycle vehicle moves along a specified path, with the direction of the vehicle being always tangent to this given path from a initial pose to a target pose.} The angular velocity $\omega\in \mathbb{R}$, the path velocity $v\in \mathbb{R}$, and the path curvature $\kappa\in \mathbb{R}$ have the following relationship \cite{Mod_car}:
\begin{equation}\label{kwv}
\omega =\kappa v.
\end{equation}
For the specified path, the curvature could be computed as a function of the path coordinate as follows:
\begin{equation}\label{k}
\kappa:s\in[0,s_e]\rightarrow \kappa(s)\in\mathbb{R},
\end{equation}
where the scalar $s$ is the path coordinate.

According to (\ref{kwv}) and (\ref{k}), the kinematic model of the vehicle is described as
\begin{equation} \label{kin}
\bm u = \bm M(s) \dot s,
\end{equation}
where $\bm u = [\omega \ v]^\mathrm{T}$, and $\bm M(s) = [\kappa(s) \ 1]^\mathrm{T}$.
Taking the time derivative of (\ref{kin}) yields that
\begin{equation}\label{kin-1}
\bm {\dot u} = \bm M(s)\ddot s +\bm M_s (s)\dot s^2,
\end{equation}
where $\bm{\dot u} \!=\! [\dot w \ \dot v]^\mathrm{T}$, $\bm M_{s}(s)=[\kappa_s \ 0]^\mathrm{T}$, $\kappa_s=d\kappa/ds$. The scalars $\dot v, \dot w$ are the path acceleration and angular acceleration, respectively.
Velocity and acceleration constraints are given as \begin{align}
-\bm v_{max} \leq \bm u \leq \bm v_{max} \label{cv1}, \\
-\bm a_{max} \leq \bm{\dot u} \leq \bm a_{max}, \label{ca1}
\end{align}
where $\bm v_{max}\in \mathbb{R}^2$ and $\bm a_{max}\in \mathbb{R}^2 $ are constant vectors representing the velocity boundary and the acceleration boundary, respectively. {These vector inequalities (\ref{cv1})-(\ref{ca1}) should be interpreted componentwise.}

In order to guarantee acceleration constraints, substituting (\ref{kin-1}) into (\ref{ca1}) yields that
\begin{equation} \label{ABC}
\bm A(s)\ddot s + \bm B(s)\dot s^2 + \bm C(s) \leq \bm 0,
\end{equation}
where $\bm A(s) = [\bm M(s)^\mathrm{T} \, -\bm M(s)^\mathrm{T}]^\mathrm{T}$, $\bm B(s) = [\bm M_s(s)^\mathrm{T} \ -\bm M_s(s)^\mathrm{T}]^\mathrm{T}$ and $\bm C (s) = -[\bm a^\mathrm{T}_{max} \ \bm a^\mathrm{T}_{max}]^\mathrm{T}$, which are all $4\times 1$ vectors.

In order to guarantee velocity constraints, substituting (\ref{kin}) into (\ref{cv1}) yields that
\begin{equation} \label{AD}
\bm A(s)\dot s + \bm D(s) \leq \bm 0,
\end{equation}
where $\bm A(s) = [\bm M(s)^\mathrm{T} \; -\bm M(s)^\mathrm{T}]^\mathrm{T}$ and $\bm D (s) = -[\bm v^\mathrm{T}_{max} \; \bm v^\mathrm{T}_{max}]^\mathrm{T}$, which are all $4\times 1$ vectors.

The velocity limit curve $MVC(s)$ considering acceleration constraints is computed with (\ref{MVC}) and (\ref{ABC}), and the limit curve $V(s)$ considering velocity constraints is computed with (\ref{V}) and (\ref{AD}). When both acceleration and velocity constraints are taken into account, the maximum velocity curve is obtained as $MVC^*$ by fusing $MVC(s)$ and $V(s)$ with (\ref{MVCSTAR}).

\subsection{Results}

As shown in Fig. \ref{path}, the specified blue path is the following cubic B\'{e}zier curve:
\begin{equation}
\begin{split}
x= & (1\!- \! \lambda)^3x_0+3(1\! - \! \lambda)^2\lambda x_1
    +3(\lambda^2\! -\! \lambda^3)x_2+\lambda^3x_3, \\
y= &(1\! -\! \lambda)^3y_0+3(1\!-\!\lambda)^2\lambda y_1
   +3(\lambda^2\!-\! \lambda^3)y_2+\lambda^3y_3,
\end{split}\nonumber
\end{equation}
where the position of the vehicle is $(x [m],y [m])$, the points $(x_i[m],y_i[m])$, $i\in[0,3]$ are path control points, and the scalar $\lambda \in [0,1]$ is the path parameter. {The $\lambda$ and $s$ obey a nonlinear scaling relation.}
The starting and terminal path velocity $\dot s_{0} = \dot s_e  = 0 [m/s]$. {This cubic B\'{e}zier curve is used to find a path from a initial pose to a target pose so that the orientation of the unicycle vehicle is always tangent to the path, and thus the nonholonomic constraint is satisfied. The following cases show that the \emph{NI} method assigns velocity profiles to the path.}
\begin{figure}[t]
      \centering
      \includegraphics[scale=0.35]{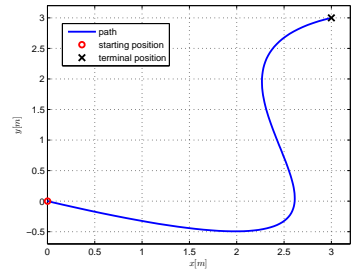}
      \caption{The specified path: cubic B\'{e}zier curve }
      \label{path}
\end{figure}

\emph{Case 1:}
This case shows that, when the velocity constraints are moderate, the \emph{NI} method \cite{NI_Shin} using $MVC^*$ generates the time-optimal trajectory. In this simulation, the velocity and acceleration constraints are set as $\bm v_{max} = [0.5rad/s \ 1.3m/s]^\mathrm{T}, \bm a_{max} = [0.05rad/s^2 \ 0.1m/s^2]^\mathrm{T}$.

In Fig. \ref{v1}, the cyan dash line represents $MVC(s)$ with acceleration constraints. When considering velocity constraints corresponding to the purple dash-dot line $V(s)$ in Fig. \ref{v1}, the maximum velocity curve is altered from $MVC$ to $MVC^*$, which is represented as the boundary between the gray (inadmissible) and blank (admissible) regions. To facilitate subsequent analysis, symbols $\#_1,\#_2$ are used to represent the two closed areas bounded by $MVC(s)$ and $V(s)$. The red and green solid lines are accelerating curves ($\beta\text{-}profiles$) and decelerating curves ($\alpha\text{-}profiles$), respectively, which
comply with \emph{Properties \ref{p1}-\ref{ps}}.
The red $\rhd$ and green $\lhd$ denote the points $sp_{\beta\rightarrow\alpha}$ and $sp_{\alpha\rightarrow\beta}$ respectively (see Section II-B).

As shown in Fig. \ref{v1}, under the curve $MVC^*$, the \emph{NI} method outputs the optimal trajectory as follows. The accelerating curve $\beta_0$, starting from $(0,\dot s_0)$, hits $MVC^*$ at $p_1$. Searching forward along $MVC^*$ from $p_1$, the first $sp_{\alpha\rightarrow\beta}$ found is $p_2$. Starting from $p_2$, the integrated backward decelerating curve $\alpha_1$ intersects $\beta_0$ at $o_1$, and the integrated forward accelerating curve $\beta_1$ hits $MVC^*$ at $p_3$. Then, searching forward along $MVC^*$ from $p_3$, the first $sp_{\alpha\rightarrow\beta}$ found is $p_4$. Starting from $p_4$, the integrated backward decelerating curve $\alpha_2$ intersects $\beta_1$ at $o_2$, and the accelerating curve $\beta_3$ hits the $MVC^*$ at $p_5$. No $sp_{\alpha\rightarrow\beta}$ is found along $MVC^*$ from $p_5$ when $s\leq s_e$. Therefore, the integrated backward $\alpha_e$, starting from $(s_e, \dot s_e)$, intersects $\beta_1$ at $o_3$. The $sp_{\beta\rightarrow\alpha}$ on $\beta_1$ is updated from $o_2$ to $o_3$, which verifies \emph{Property \ref{p3}}. Finally, the \emph{NI} method outputs the feasible and optimal trajectory: $\beta_0-\alpha_1-\beta_1-\alpha_e$.

\begin{figure}[t]
      \centering
      \includegraphics[scale=0.52]{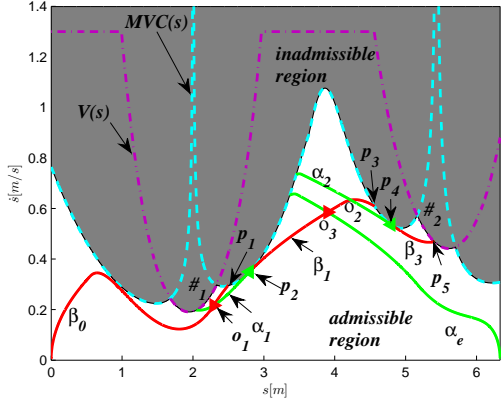}
      \caption{Case 1: The \emph{NI} method outputs the optimal trajectory with constraints $\bm v_{max} = [0.5rad/s \ 1.3m/s]^\mathrm{T}, \bm a_{max} = [0.05rad/s^2 \ 0.1m/s^2]^\mathrm{T}$}
      \label{v1}
\end{figure}

\begin{figure}[t]
      \centering
      \includegraphics[scale=0.52]{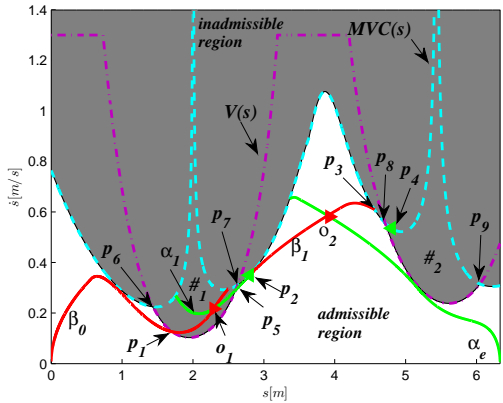}
      \caption{Case 2: The \emph{NI} method fails to output a feasible trajectory with constraints $\bm v_{max} = [0.2rad/s \; 1.3m/s]^\mathrm{T}, \bm a_{max} = [0.05rad/s^2 \; 0.1m/s^2]^\mathrm{T}$}
      \label{v2}
\end{figure}

\begin{figure}
      \centering
      \includegraphics[scale=0.52]{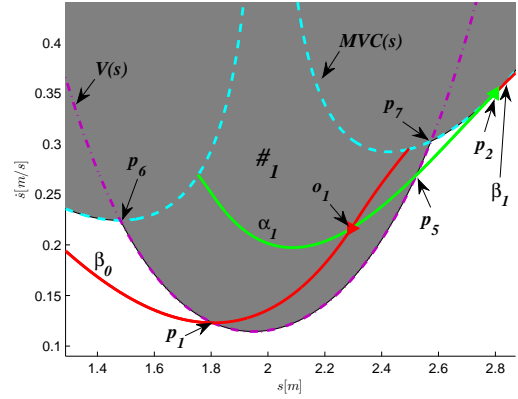}
      \caption{The enlarged view of the region $\#_1$ in Fig. \ref{v2}}
      \label{v2-1}
\end{figure}

\emph{Case 2:}
This case shows that, when the velocity constraints are too restrictive, the conditions in \emph{Property \ref{t2}} will be satisfied, and thus the \emph{NI} method \cite{NI_Shin} using $MVC^*$ fails to output a feasible trajectory. In this simulation, the velocity constraint $\bm v_{max}$ is modified as $[0.2rad/s \; 1.3m/s]^\mathrm{T}$ and the acceleration constraint $\bm a_{max}$ remains the same as that of \emph{Case 1}. Therefore, in Fig. \ref{v2}, the velocity limit curve $V(s)$ due to the velocity constraint becomes lower, and the areas of two inadmissible regions $\#_{1}, \#_2$ increase.
Meanwhile, the trajectory $T:\beta_0-\alpha_1-\beta_1-\alpha_e$, obtained by the \emph{NI} method using $MVC$, is greater than $MVC^*$ at the region $\#_{1}$, and the $sp_{\alpha\rightarrow\beta}$ on $MVC^\dagger$ is nonexistent, which indicates that the conditions of \emph{Property \ref{t2}} hold and also verifies \emph{Property \ref{t1}}.

Under the curve $MVC^*$, the \emph{NI} method fails to output a feasible trajectory, which is described as follows. For clarity, the region $\#_1$ is enlarged as shown in Fig. \ref{v2-1}.
The integrated forward curve $\beta_0$ from $(0,\dot s_0)$ hits the $MVC^*$ at $p_1$. Then, searching forward from $p_1$ along $MVC^*$, the first $sp_{\alpha\rightarrow\beta}$ found is the point $p_2$. However, starting from $p_2$, the integrated backward $\alpha_1$ hits the $MVC^*$ at $p_5$ before intersecting $\beta_0$ (the point $o_1$ is in the inadmissible region). Then, the integrated forward $\beta_1$ from $p_2$ hits the $MVC^*$ at $p_3$. No $sp_{\alpha\rightarrow\beta}$ is found along $MVC^*$ from $p_3$ when $s\leq s_e$.
Therefore, the integrated backward $\alpha_e$, starting from $(s_e, \dot s_e)$, intersects $\beta_1$ at $o_2$.
In all these procedures, the $\alpha\text{-}profiles$ on the right side of the purple dash-dot curve $MVC^\dagger$ (\overarc[0.65]{$p_6p_7$}), cannot intersect $\beta_0$. This $MVC^\dagger$ (\overarc[0.65]{$p_6p_7$}) breaks the intersection between the accelerating and decelerating curves, and causes that the final trajectory is blank between $p_1$ and $p_5$, which indicates that the \emph{NI} method fails and verifies \emph{Property \ref{t2}}.


\section{Conclusion}

This letter revisits the original version of \emph{NI} method for time-optimal trajectory planning along specified paths.
On this basis, we first summarize several known and new properties regarding switch points and accelerating/decelerating curves of the \emph{NI} method, and give corresponding mathematical proofs. Then, we provide concrete failure conditions and rigorous proofs for the property, which indicates that, in the presence of velocity constraints, the original version of \emph{NI} which only considers torque constraints may result in failure of trajectory planning tasks. Accordingly, a failure detection algorithm is given in a `run-and-test' manner. Simulation results on a unicycle vehicle are provided to verify these presented properties.

\addtolength{\textheight}{-0cm}   

\section*{Appendix}

\subsection{{Proof of Property 2}}\label{app-a}

\begin{proof}
In the $AREM$ region, an $\alpha\text{-}profile$ intersects an $\beta\text{-}profile$ at a point $(s = s_c,\dot s = \dot s_c)$. At the neighborhood of $s_c$, the slopes of the $\alpha\text{-}profile$ and $\beta\text{-}profile$ satisfy the inequality $k_\alpha < k_\beta $ and the Lipschitz condition \cite{NI_Bobrow}. Therefore, based on \emph{Comparison Theorem}\cite{Math1} {(Let $y, z$ be solutions of the differential equations $\dot y=F(x,y), \dot z=G(x,z)$. If $F(x,y) < G(x,z),x\in [a,b]$, the function $F$ or $G$ satisfies a Lipschitz condition, and $y(a)=z(a)$, then $y(x) < z(x),x\in (a,b]$)}, it is proven that the $\alpha\text{-}profile$ is greater than the $\beta\text{-}profile$ in the left neighborhood of $s_c$, but less than the $\beta\text{-}profile$ in the right neighborhood of $s_c$.
\end{proof}

\subsection{{Proof of Property 3}}\label{app-b}

\begin{proof}
This property is proven by contradiction.
Assume that an $\alpha\text{-}profile$ is tangent to another $\beta\text{-}profile$ in the $AREM$ region. Then, on the tangent point, the slope $k_{\alpha}$ of the $\alpha\text{-}profile$ is equal to $k_\beta$ of the $\beta\text{-}profile$, which contradicts with the inequality $k_\alpha < k_\beta$ in the $AREM$ region. Thus, the assumption is invalid and the property is proven.
\end{proof}

\subsection{{Proof of Property 4}}\label{app-c}

\begin{proof}
In terms of the number $m$ of intersection points, there are totally two cases: $m=1,  m > 1$.

\emph{Case 1:} $m=1$. There is only one intersection point, so the point $sp_{\beta\rightarrow\alpha}$ on the $\beta\text{-}profile$ is $X_1$. This property holds for $m=1$.

\emph{Case 2:} $m>1$. There are $m$ intersection points as Fig. \ref{con5}.
        In terms of path coordinate, the intersection point $X_i$ is less than $ X_j$, $1 \leq i < j \leq m$.
        Each $X_i$ has one corresponding decelerating curve $\alpha_i$ and one switch point $Y_i$.
        According to \emph{Property \ref{p1}}, $Y_i$ is at the right side of $Y_j$, $i<j$.
        If $X_i$, $i > 1$ is chosen as $sp_{\beta\rightarrow\alpha}$ on $\beta^*$,
        then, starting from $X_i$, the trajectory consisting of $\alpha\text{-}profiles$ and $\beta\text{-}profiles$ cannot leave the region $D$, which is enclosed by $\alpha_1$, $\beta^*$ and $MVC$, across $\alpha_1$ due to \emph{Properties \ref{p1}-\ref{ps}}. Thus, $X_1$ is chosen as $sp_{\beta\rightarrow\alpha}$ on $\beta^*$, which can aid the trajectory to leave the region $D$ along $\alpha_1$ and go on extending to the right side of $Y_1$ with $\beta_1$.
        This property holds for $m>1$. In summary, \emph{Property \ref{p3}} holds.
\end{proof}

\subsection{{Proof of Property 5}}\label{app-d}

\begin{proof}
Due to (\ref{MVCtt}), $MVC^\dagger$ is less than $MVC$. According to the definition of $AR$, the inequality $\alpha(s,\dot s) <  \beta(s,\dot s)$ holds on $MVC^\dagger$. Then, based on the facts $k_\alpha = \alpha(s,\dot s)/ \dot s, k_\beta = \beta(s,\dot s)/ \dot s$, the inequality $k_\alpha < k_\beta$ also holds on $MVC^\dagger$, which violates $k_\alpha = k_{\beta}$ in the definition of tangent switch points (see $sp_{\alpha\rightarrow\beta}$ in Section II-B). Therefore, tangent switch points on $MVC^\dagger$ are nonexistent.
\end{proof}



\begin{thebibliography}{99}

\bibitem{offline} E. Barnett and C. Gosselin, ``Time-optimal trajectory planning of cable-driven parallel mechanisms for fully-specified paths with G1 discontinuities,'' \emph{Journal of Dynamic Systems Measurement and Control}, vol. 137, no. 7, pp. 603-617, 2013.

\bibitem{OWMR_Kim} K. Kim and B. Kim, ``Minimum-time trajectory for three-wheeled omnidirectional mobile robots following a bounded-curvature path with a referenced heading profile,'' \emph{IEEE Transactions on Robotics}, vol. 27, no. 4, pp. 800-808, 2011.

\bibitem{No_sm2} S. Macfarlane and E. Croft, ``Jerk-bounded manipulator trajectory
planning: design for real-time applications,'' \emph{IEEE Transactions on Robotics and Automation},
vol. 19, no. 1, pp. 42-52, 2003.

\bibitem{Dp1} K. Shin and N. McKay, ``Selection of near-minimum time geometric paths for robotic manipulators,'' \emph{IEEE Transactions on Automatic Control}, vol. 31, no. 6, pp. 501-511, 1986.

\bibitem{Dp2} K. Shin and N. McKay, ``A dynamic programming approach to trajectory planning of robotic manipulators,'' \emph{IEEE Transactions Automatic Control}, vol. 31, no. 6, pp. 491-500, 1986.

\bibitem{Dp3} S. Singh and M. Leu, ``Optimal trajectory generation for robotic manipulators using dynamic programming,'' \emph{Journal of Dynamic Systems Measurement and Control}, vol. 109, no. 2, pp. 88-96, 1987.

\bibitem{Convex1} D. Verscheure, B. Demeulenaere, J. Swevers, J. Schutter, and
M. Diehl, ``Time-optimal path tracking for robots: A convex optimization approach,'' \emph{IEEE Transactions on Automatic Control}, vol. 54, no. 10, pp. 2318-2327, 2009.

\bibitem{Convex2} K. Hauser, ``Fast interpolation and time-optimization with contact,'' \emph{International Journal of Robotics Research}, vol. 33, no. 9, pp. 1231-1250, 2014.

\bibitem{Convex3} F. Debrouwere, W. Loock, G. Pipeleers, Q. Dinh, M. Diehl, J. Schutter and J. Swevers, ``Time-optimal path following for robots
with convex-concave constraints using sequential convex programming,'' \emph{IEEE Transactions on Robotics}, vol. 29, no. 6, pp. 1485-1495, 2013.

\bibitem{NI_Shin} K. Shin and N. Mckay, ``Minimum-time control of robotic manipulators with geometric path constraints,'' \emph{IEEE Transactions on Automatic Control}, vol. 30, no. 6, pp. 531-541, 1985.

\bibitem{NI_Bobrow} J. Bobrow, S. Dubowsky and J. Gibson, ``Time-optimal control of robotic manipulators along specified paths,'' \emph{International Journal of Robotics Research}, vol. 4, no. 3, pp. 3-17, 1985.



\bibitem{NI_Improve} J. Slotine and H. Yang, ``Improving the efficiency of time-optimal path-following algorithms,'' \emph{IEEE Transactions on Robotics and Automation}, vol. 5, no. 1, pp. 118-124, 1989.

\bibitem{NI_Quang} Q. Pham, ``A general, fast, and robust implementation of the time-optimal path parameterization algorithm,'' \emph{IEEE Transactions on Robotics}, vol. 30, no. 6, pp. 1533-1540, 2014.

\bibitem{NI_Pfeiffer} F. Pfeiffer and R. Johanni, ``A concept for manipulator trajectory planning,'' \emph{IEEE Journal on Robotics and Automation}, vol. 3, no. 2, pp. 115-123, 1987.

\bibitem{NI_Shiller} Z. Shiller, ``On singular time-optimal control along specified paths,'' \emph{IEEE Transactions on Robotics and Automation}, vol. 10, no. 4, pp. 561-566, 1994.


\bibitem{NI_Quang2} Q. Pham and O. Stasse, ``Time-optimal path parameterization for redundantly actuated robots: A numerical integration approach,'' \emph{IEEE/ASME Transactions on Mechatronics}, vol. 20, no. 6, pp. 3257-3263, 2015.

\bibitem{NI_Beh} S. Behzadipour and A. Khajepour, ``Time-optimal trajectory planning in cable-based manipulators,'' \emph{IEEE Transactions on Robotics}, vol. 22, no. 3, pp. 559-563, 2006.

\bibitem{No_sm1} T. Kunz and M. Stilman, ``Time-optimal trajectory generation for path following with bounded acceleration and velocity,'' \emph{Robotics: Science and Systems}, vol. 8, no. 12, pp. 9-13, 2012.

\bibitem{NI_spacecraft} H. Nguyen and Q. Pham, ``Time-optimal path parameterization of rigid-body motions: Applications to spacecraft reorientation,''\emph{ Journal of Guidance Control and Dynamics}, vol. 39, no. 7, pp. 1667-1671, 2016.

\bibitem{NI_human} Q. Pham and Y. Nakamura, ``Time-optimal path parameterization for critically dynamic motions of humanoid robots,'' \emph{Proceedings of the 2012 IEEE International Conference on Humanoid Robots}, pp. 165-170, 2012.


\bibitem{NI_Zlajpah} L. \v{Z}lajpah, ``On time optimal path control of manipulators with bounded joint velocities and torques,'' \emph{Proceedings of the 1996 IEEE International Conference on Robotics and Automation}, pp. 1572-1577, 1996.

\bibitem{NI_Lamiraux} F. Lamiraux and J. Laumond, ``From paths to trajectories for multi-body mobile robots,'' \emph{Proceedings of the 5th International Symposium on Experimental Robotics}, pp. 301-309, 1998.

\bibitem{Kal} T. Kalm\'{a}r-Nagy, ``Real-time trajectory generation for omni-directional vehicles by constrained dynamic inversion,'' \emph{Mechatronics}, vol. 35, pp. 44-53.

\bibitem{AVP} Q. Pham, S. Caron and Y. Nakamura, ``Kinodynamic planning in the configuration space via velocity interval propagation,'' \emph{Proceedings of Robotics: Science and Systems}, 2013.

\bibitem{Mod_car}
C. Bianco and M. Romano, ``Optimal velocity planning for autonomous vehicles considering curvature constraints,'' \emph{Proceedings of
the 2007 IEEE International Conference on Robotics and Automation}, pp. 2706-2711, 2007.

\bibitem{Math1} G. Birkhoff and G. Rota, ``Ordinary differential equations,'' \emph{Ussr Computational Mathematics Physics}, vol. 4, no. 4, pp. 230-232, 1964.

\end{thebibliography}
\end{document}